\documentclass[final,12pt]{clear2025} 

\usepackage[sc]{mathpazo}
\usepackage{amsfonts,latexsym,amsmath,stmaryrd,dsfont}
\usepackage{multido}
\usepackage{graphicx}
\usepackage{fancybox}
\usepackage[normalem]{ulem}
\usepackage{amscd}
\usepackage{enumitem}
\usepackage[utf8]{inputenc}
\usepackage{multicol}
\usepackage{tikz}
\usepackage{mathrsfs}
\usepackage{algorithm2e}
\usetikzlibrary{arrows}
\definecolor{aqaqaq}{rgb}{0.6274509803921569,0.6274509803921569,0.6274509803921569}

\usepackage[all,cmtip]{xy}
\xyoption{frame}

\usepackage{algorithmicx}
\usepackage{algorithm}
\usepackage{algpseudocode}
\usepackage{listings}
\lstdefinelanguage{pseudocode}{
	morekeywords={if, return, Input, Output, True, False, For, not}
}
\lstdefinestyle{pseudocode}{
	language=pseudocode,
	keywordstyle=\bfseries
}

\def\eps{\varepsilon}
\def\cA{{\mathcal A}}

\def\cF{{\mathcal F}}
\def\cG{{\mathcal G}}
\def\cH{{\mathcal H}}

\def\cS{{\mathcal S}}

\def\cX{{\mathcal X}}

\def\nset{\mathbb N}

\def\bA{\mathbf{A}}
\def\bR{\mathbf{R}}
\def\tbR{\widetilde{\mathbf{R}}}
\def\bX{\mathbf{X}}
\def\bY{\mathbf{Y}}
\def\bV{\mathbf{V}}
\def\bW{\mathbf{W}}

\def\obX{{\overline{\mathbf{X}}}}

\def\tcF{{\widetilde{\mathcal{F}}}}

\DeclareMathOperator{\dist}{dist}
\DeclareMathOperator{\layer}{layer}
\DeclareMathOperator{\Do}{do}
\DeclareMathOperator{\An}{An}
\DeclareMathOperator{\Pa}{Pa}
\DeclareMathOperator{\ID}{ID}
\DeclareMathOperator{\tmin}{t_{min}}
\DeclareMathOperator{\tmax}{t_{max}}

\title{Causal Identification in Time Series Models}
\usepackage{times}

\clearauthor{
 \Name{Erik Jahn} \Email{ejahn@caltech.edu}\\
 \Name{Karthik Karnik} \Email{sathkar2009@gmail.com}\\
 \Name{Leonard J. Schulman} \Email{schulman@caltech.edu}\\
 \addr Engineering and Applied Science, California Institute of Technology, Pasadena CA 91125, USA.
}

\begin{document}

\maketitle

\begin{abstract}
  In this paper, we analyze the applicability of the Causal Identification algorithm to causal time series graphs with latent confounders. Since these graphs extend over infinitely many time steps, deciding whether causal effects across arbitrary time intervals are identifiable appears to require computation on graph segments of unbounded size. Even for deciding the identifiability of intervention effects on variables that are close in time, no bound is known on how many time steps in the past need to be considered. We give a first bound of this kind that only depends on the number of variables per time step and the maximum time lag of any direct or latent causal effect. More generally, we show that applying the Causal Identification algorithm to a constant-size segment of the time series graph is sufficient to decide identifiability of causal effects, even across unbounded time intervals. 
\end{abstract}


\section{Introduction}

Causal inference on time series data has many important applications across fields such as economics, finance, earth and climate science \citep{Runge-earth-science, Moraffah21, ZHANG2014}. In these domains, conducting large-scale experiments to determine causal effects is often impractical or impossible, but large observational time-series data sets are frequently available. Developing statistical methods that leverage these datasets to infer causal relationships is thus an important problem. Causal time series graphs (also called Dynamic Bayesian Networks) provide a well-studied model for this purpose \citep{Assaad22}. The vertices of these graphs represent a set of observable random variables recurring at each non-negative integer time, and the edges represent direct causal effects or latent confounding. The key assumption for these graphs is that the causal structure is time-invariant, that is, edges are invariant under time shifts. However, the joint distributions of sets of variables need not be time-invariant. Numerous methods have been explored for causal discovery and causal feature selection in these graphs, both in the setting without latent variables \citep{runge_PCMCI, Pfister19} and with latent variables \citep{malinsky-spirtes18, mastakouri21, Runge-LPCMCI}.

In many cases, models without latent variables can be overly simplistic. Particularly for time series data, latent confounding may not only result from unobserved covariates, but can also arise when the time resolution of the observations is too coarse, even when all relevant variables are observed \citep{PetersJS17, Runge-survey}. In this paper we focus on causal time series graphs with latent variables. In this setting, even when a causal graph is discovered, the problem of causal identification remains. It is well-known that some causal effects in certain graphs may be unidentifiable \citep{Pea09}. For finite graphs the question of causal identifiability is answered precisely by the Causal ID algorithm \citep{tian_general_2002,shpitserP06,huangV09,shpitser08}. However, in infinite periodic graphs representing a time series model, causal identification is less well-studied. The only existing work on causal identifiability in these graphs that we are aware of does not address the challenges described below \citep{Blondel_Arias_Gavaldà_2016}. Other work on causal identifiability in time series settings is based on Granger causality \citep{Eichler07,Eichler_Didelez_2009}, which is known to have certain limitations due to violations of its assumptions \citep{PetersJS17}. One might argue that the direction of time allows the standard ID algorithm to be applied even to infinite time series graphs. Indeed, one can show that identifying the causal effect of a variable $X$ on $Y$ requires considering only the section of the graph up to the time at which $Y$ occurs, see Section~\ref{sec:results}. However, the following problems remain: 
\begin{enumerate}
    \item While the causal effect of $X$ on $Y$ is not affected by variables that lie in the future of $Y$, there may be chains of confounding variables in the past that obstruct identifiability (see Section~\ref{sec:results} for an example). Hence, computing the identifiability of $P(Y| \Do X)$ would naively require running the Causal ID algorithm on the entire past of the time series, from its initialization up to the occurrence of $Y$. Consider, for instance, a causal time series graph modeling Earth's climate system. Using the ID algorithm to check the identifiability of the causal effect between two variables in the near-present, would require computation scaling with the length of Earth's history, which is clearly infeasible. For causal discovery in such settings, one often just considers the past up to a maximum time interval $\tau_{\max}$. However, $\tau_{\max}$ is typically chosen empirically and this may distort causal effects (see \citet{Runge-LPCMCI}, Section 2.3). A priori, the effect of such a cut-off on causal identification is unclear.

    \item  Suppose we are interested in the causal effect between recurring variables $X$ and $Y$ but do not know the time span over which this effect might occur. A causal time series graph employed to investigate this effect should have the property that the causal effect of $X_t$ on $Y_{t+\Delta}$ is identifiable for all time steps $\Delta \geq 0$ in the future. However, the Causal ID algorithm cannot decide in finite computing time whether a graph has this property. The only option seems to be to restrict the question of identifiability up to a maximum time step $\Delta_{\max}$, which would still require increasing computing time in $\Delta_{\max}$.
\end{enumerate}

In this paper, we provide solutions to both problems through the following result. 

\begin{theorem}\label{thm:main}
    Let $\cG$ be a periodic causal graph and let $\bX, \bY \subseteq \bV(\cG)$ be disjoint subsets of variables. Let $t_{\min}$ be the smallest time index of a variable in $\bX$ and $t_{\max}$ the largest time index of a variable in $\bY$. Then, there exists a constant $C$ such that the following holds:
    \begin{enumerate}
        \item Running the Causal ID algorithm on the segment of $\cG$ with time indices between $t_{\min} - C$ and $t_{\max}$ suffices to determine identifiability of the causal effect of $\bX$ on $\bY$. \label{item:a}
        \item If the causal effect of $\bX$ on $\bY$ is unidentifiable, then there is a time shift $\Delta$ such that the time difference between $\bX$ and $\bY_{-\Delta}$ is at most $C$ and the causal effect of $\bX$ on $\bY_{-\Delta}$ is also unidentifiable. \label{item:b}
    \end{enumerate}   
\end{theorem}

See Section~\ref{sec:preliminaries} for formal definitions of the notation used above. Our result implies that identifiability of the causal effect between any two sets of variables can be decided in computing time that scales with $C$ but not with the size of the time interval between these variables or the distance to the initialization time of the causal model, see Figure~\ref{fig:main}. Here, $C$ depends only on the number of variables per time step and the maximum time lag of any direct or latent causal effect. In Section~\ref{sec:results} we show specific upper and lower bounds on the scaling of $C$. While our upper bound exhibits exponential growth, it is only a first generic bound, applicable to any periodic graph structure. For specific graphs, $C$ could potentially be much smaller. In fact, we are only able to construct graphs that require a linear scaling of $C$ in the number of variables per time step, and we present this construction together with the proofs of our results in Section~\ref{sec:results}.

\usetikzlibrary{arrows.meta}
\usetikzlibrary{decorations.pathreplacing}
\begin{figure}[ht]
\centering
\begin{tikzpicture}[line cap=round,line join=round,>=Latex,x=1cm,y=1cm, scale=0.9]

\draw [fill=black] (-3,6) circle (2pt);
\draw [fill=black] (-3,5) circle (2pt);
\draw [fill=black] (-3,4) circle (2pt);
\draw [fill=black] (-2,6) circle (2pt);
\draw [fill=black] (-2,5) circle (2pt);
\draw [fill=black] (-2,4) circle (2pt) node[yshift=-1em]{Initializiation};
\draw [fill=black] (-1,6) circle (2pt);
\draw [fill=black] (-1,5) circle (2pt);
\draw [fill=black] (-1,4) circle (2pt);
\draw [fill=black] (0,6) circle (2pt);
\draw [fill=black] (0,5) circle (2pt);
\draw [fill=black] (0,4) circle (2pt);
\draw [fill=black] (1,6) circle (2pt);
\draw [fill=black] (1,5) circle (2pt);
\draw [fill=black] (1,4) circle (2pt);
\draw (2,5) node {$\ldots$};
\draw [fill=black] (3,6) circle (2pt);
\draw [fill=black] (3,5) circle (2pt);
\draw [fill=black] (3,4) circle (2pt);
\draw [fill=black] (4,6) circle (2pt);
\draw [fill=black] (4,5) circle (2pt);
\draw [fill=black] (4,4) circle (2pt);
\draw [fill=black] (5,6) circle (2pt) node[yshift=1em]{$X$};
\draw [fill=black] (5,5) circle (2pt);
\draw [fill=black] (5,4) circle (2pt);
\draw [fill=black] (6,6) circle (2pt);
\draw [fill=black] (6,5) circle (2pt);
\draw [fill=black] (6,4) circle (2pt);
\draw [fill=black] (7,6) circle (2pt);
\draw [fill=black] (7,5) circle (2pt);
\draw [fill=black] (7,4) circle (2pt) node[yshift=-1em, xshift=1em] {$Y_{-\Delta}$};
\draw (8,5) node {$\ldots$};
\draw [fill=black] (9,6) circle (2pt);
\draw [fill=black] (9,5) circle (2pt);
\draw [fill=black] (9,4) circle (2pt);
\draw [fill=black] (10,6) circle (2pt);
\draw [fill=black] (10,5) circle (2pt);
\draw [fill=black] (10,4) circle (2pt);
\draw [fill=black] (11,6) circle (2pt);
\draw [fill=black] (11,5) circle (2pt);
\draw [fill=black] (11,4) circle (2pt);
\draw [fill=black] (12,6) circle (2pt);
\draw [fill=black] (12,5) circle (2pt);
\draw [fill=black] (12,4) circle (2pt);
\draw [fill=black] (13,6) circle (2pt);
\draw [fill=black] (13,5) circle (2pt);
\draw [fill=black] (13,4) circle (2pt) node[yshift=-1em]{$Y$};
\draw (14,5) node {$\ldots$};

\draw [->, line width=1pt] (-3,6)-- (-2,5);
\draw [->, line width=1pt] (-3,5)-- (-2,4);
\draw [->, line width=1pt] (-3,4)-- (-2,4);
\draw [<->, line width=1pt, dash pattern=on 3pt off 3pt] (-3,4) -- (-2,6);

\draw [->, line width=1pt] (-2,6)-- (-1,5);
\draw [->, line width=1pt] (-2,5)-- (-1,4);
\draw [->, line width=1pt] (-2,4)-- (-1,4);
\draw [<->, line width=1pt, dash pattern=on 3pt off 3pt] (-2,4) -- (-1,6);

\draw [->, line width=1pt] (-1,6)-- (0,5);
\draw [->, line width=1pt] (-1,5)-- (0,4);
\draw [->, line width=1pt] (-1,4)-- (0,4);
\draw [<->, line width=1pt, dash pattern=on 3pt off 3pt] (-1,4) -- (0,6);

\draw [->, line width=1pt] (0,6)-- (1,5);
\draw [->, line width=1pt] (0,5)-- (1,4);
\draw [->, line width=1pt] (0,4)-- (1,4);
\draw [<->, line width=1pt, dash pattern=on 3pt off 3pt] (0,4) -- (1,6);

\draw [->, line width=1pt] (3,6)-- (4,5);
\draw [->, line width=1pt] (3,5)-- (4,4);
\draw [->, line width=1pt] (3,4)-- (4,4);
\draw [<->, line width=1pt, dash pattern=on 3pt off 3pt] (3,4) -- (4,6);

\draw [->, line width=1pt] (4,6)-- (5,5);
\draw [->, line width=1pt] (4,5)-- (5,4);
\draw [->, line width=1pt] (4,4)-- (5,4);
\draw [<->, line width=1pt, dash pattern=on 3pt off 3pt] (4,4) -- (5,6);

\draw [->, line width=1pt] (5,6)-- (6,5);
\draw [->, line width=1pt] (5,5)-- (6,4);
\draw [->, line width=1pt] (5,4)-- (6,4);
\draw [<->, line width=1pt, dash pattern=on 3pt off 3pt] (5,4) -- (6,6);

\draw [->, line width=1pt] (6,6)-- (7,5);
\draw [->, line width=1pt] (6,5)-- (7,4);
\draw [->, line width=1pt] (6,4)-- (7,4);
\draw [<->, line width=1pt, dash pattern=on 3pt off 3pt] (6,4) -- (7,6);

\draw [->, line width=1pt] (9,6)-- (10,5);
\draw [->, line width=1pt] (9,5)-- (10,4);
\draw [->, line width=1pt] (9,4)-- (10,4);
\draw [<->, line width=1pt, dash pattern=on 3pt off 3pt] (9,4) -- (10,6);

\draw [->, line width=1pt] (10,6)-- (11,5);
\draw [->, line width=1pt] (10,5)-- (11,4);
\draw [->, line width=1pt] (10,4)-- (11,4);
\draw [<->, line width=1pt, dash pattern=on 3pt off 3pt] (10,4) -- (11,6);

\draw [->, line width=1pt] (11,6)-- (12,5);
\draw [->, line width=1pt] (11,5)-- (12,4);
\draw [->, line width=1pt] (11,4)-- (12,4);
\draw [<->, line width=1pt, dash pattern=on 3pt off 3pt] (11,4) -- (12,6);

\draw [->, line width=1pt] (12,6)-- (13,5);
\draw [->, line width=1pt] (12,5)-- (13,4);
\draw [->, line width=1pt] (12,4)-- (13,4);
\draw [<->, line width=1pt, dash pattern=on 3pt off 3pt] (12,4) -- (13,6);

\draw [decorate,decoration={brace,amplitude=5pt,mirror,raise=2ex}]
  (3,4) -- (5,4) node[midway,yshift=-2em]{$C$ layers};
  \draw [decorate,decoration={brace,amplitude=5pt,mirror,raise=2ex}]
  (5,4) -- (7,4) node[midway,yshift=-2em]{$C$ layers};

\end{tikzpicture}

\caption{A priori, deciding if the causal effect of $X$ on $Y$ is identifiable requires running the causal ID algorithm on the entire part of the time series graph shown above - from its initialization up to the layer containing $Y$. Our results show that instead computation on a constant-size section around $X$ (i.e. just the middle part) suffices, using a shifted version $Y_{-\Delta}$ of $Y$.} \label{fig:main} 
\end{figure}
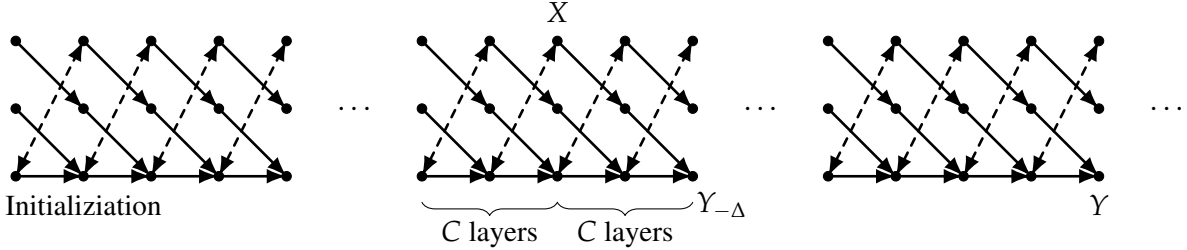

\section{Preliminaries} \label{sec:preliminaries}
First, we define notions related to the causal time series graphs that we consider and fix notation that we use throughout this paper. To clearly distinguish sets of variables from singletons, we use bold letters to denote sets. Graphs are denoted by calligraphic letters. All graphs in this paper are acyclic directed mixed graphs (ADMGs), i.e. graphs with both directed edges representing direct causal effects and bidirected edges representing confounding variables.

\begin{definition}[periodic causal graph] A periodic causal graph of width $w$ is an acyclic directed mixed graph with vertices labeled by  $X_{i,t}$ for $0 \leq i \leq w-1$ and $t \in \nset$. We say, the vertex $X_{i,t}$ is in row $i$ and column, or time, $t$ and denote the set of vertices at time $t$ by $\bX_{*,t}$.

The graph is periodic in that $(X_{i,t},X_{i',t'})$ is a directed (resp.\ bidirected) edge if and only if the same is true of $(X_{i,t+1},X_{i',t'+1})$ and directed edges do not go backwards in time, i.e. each directed edge $(X_{i,t},X_{i',t'})$ must have $t\leq t'$ (allowing for contemporaneous edges). 

 \end{definition}

 \begin{definition}[segments]
     We denote the segment of a periodic causal graph $\cG$ that is induced by the vertices $X_{i,t}$ with $0 \leq i \leq w-1$ and $T \leq t \leq T'$, by $\cG[T, T']$.
 \end{definition}

 \begin{definition}[latency]
     A periodic causal graph is of latency $L$ if for any directed or bidirected edge $(X_{i,t},X_{i',t'})$, we have $|t-t'|\leq L$. 
 \end{definition}

 \begin{definition}[distance between sets]
    Let $\cG$ be a periodic causal graph with vertex set $\bV(\cG)$. For a vertex $V \in \bV(\cG)$, define $\layer(V) = t$, if $V = X_{i,t}$. For a set $\bX \subseteq \bV(\cG)$, we define its minimum layer and maximum layer by
    \begin{align*}
        \tmin(\bX) &= \min \{\layer(V) | V \in \bX\}\\
        \tmax(\bX) &= \max \{\layer(V) | V \in \bX\}.
    \end{align*}
    Finally, the distance between two sets $\bX, \bY \subseteq \bV(\cG)$ is given by 
    \begin{align*}
        \dist(\bX, \bY) = \min \{|\layer(V_1) - \layer(V_2)| \mid V_1 \in \bX, V_2 \in \bY\}.
    \end{align*}
\end{definition}

\begin{definition}[time shifts]
    Let $\cG$ be a periodic causal graph and $\bX \subseteq \bV(\cG)$ with $\tmin(\bX) \geq \Delta$ for some $\Delta \in \nset$. Then, we define the shifted sets 
    \begin{align*}
        \bX_{+\Delta} &= \{X_{i,t+\Delta} | X_{i,t} \in \bX\},\\
        \bX_{-\Delta} &= \{X_{i,t-\Delta} | X_{i,t} \in \bX\}.
    \end{align*}
\end{definition}

\begin{figure*}
 \[ \xymatrix{
X_{0,0} \ar@{->}[d] \ar@{->}[r]&X_{0,1} \ar@{->}[d] \ar@{->}[r] &X_{0,2} \ar@{->}[d] \ar@{->}[r] & X_{0,3} \ar@{->}[d] \ar@{->}[r] &X_{0,4} \ar@{->}[d] \ar@{->}[r] & \ldots \\
X_{1,0} \ar@{->}[ru] \ar@{<-->}[rd]&  X_{1,1} \ar@{->}[ru] \ar@{<-->}[rd]&  X_{1,2} \ar@{->}[ru] \ar@{<-->}[rd]& X_{1,3} \ar@{->}[ru] \ar@{<-->}[rd]&  X_{1,4} \ar@{->}[ru] \ar@{<-->}[rd]& \ldots \\
X_{2,0} \ar@{->}[ru] \ar@{->}[r]& X_{2,1} \ar@{->}[ru] \ar@{->}[r]& X_{2,2} \ar@{->}[ru] \ar@{->}[r]& X_{2,3} \ar@{->}[ru] \ar@{->}[r]& X_{2,4} \ar@{->}[ru] \ar@{->}[r] & \ldots } \] 
\caption{Periodic graph of width $3$ and latency $1$. Directed edges are solid, bidirected dashed.} \label{AXexample}
\end{figure*}
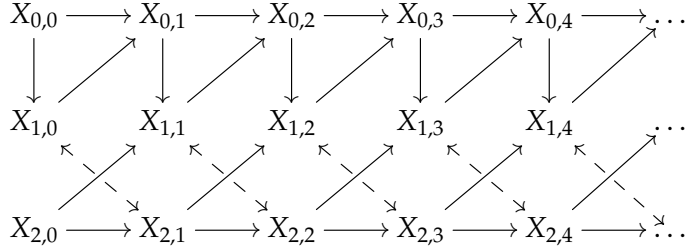

See Figure~\ref{AXexample} for an example of a periodic causal graph. A periodic causal graph $\cG$ represents a causal model, which is given by a set of equations
\begin{align*}
    X_{i,t} = f_{i,t}(\Pa(X_{i,t}, \cG), \eps_{i,t}).
\end{align*}
Here, $\Pa(X_{i,t}, \cG)$ denotes the parents of $X_{i,t}$ in $\cG$ and the $\eps_{i,t}$ are random noise variables that may be correlated for variables $X_{i,t}$ that are connected by a bi-directed edge, and are independent otherwise. Note that we do not require the functions $f_{i,t}$ to be the same for different time steps. For two sets of variables $\bX, \bY \subseteq \bV(\cG)$, $P(\bY | \Do \bX = \mathbf{x})$ denotes the probability distribution of the variables $\bY$, when the causal model is generated using the above equations, but with the variables in $\bX$ set to values given by $\mathbf{x}$. This distribution quantifies the \emph{causal effect} of setting $\bX = \mathbf{x}$ on $\bY$. In what follows, we usually do not specify the assignment of values $\mathbf{x}$. We say that $P(\bY | \Do \bX)$ is \emph{identifiable} if one can compute it, knowing only the causal graph $\cG$ and the total probability distribution of the variables $\bV(\cG)$ but not the specific functions $f_{i,t}$ or noise terms $\eps_{i,t}$ giving rise to the causal model. For a more detailed introduction into graphical causal models and do-calculus, see \citet{Pea09}, and see \citet{Blondel_Arias_Gavaldà_2016} for time series graphs specifically.

We will make use of the well-known characterization of identifiability for causal effects in finite causal graphs, developed by \citet{tian_general_2002, shpitserP06,huangV09}. The following definitions and results are according to \citet{shpitser08}: 

\begin{definition}[C-component] 
An ADMG $\cG$ where any pair of vertices is connected by a path of bidirected edges is called a C-component (confounded component).
\end{definition}

\begin{definition}[forest]
    An ADMG $\cG$ in which each vertex has at most one child is called a forest. The set of vertices $\bR \subseteq \bV(\cG)$ that do not have any children is called the set of roots of $\cG$.
\end{definition}

\begin{definition}[C-forest]
    An ADMG $\cG$ that is both a C-component and a forest is called a C-forest. 
\end{definition}

Now, for a vertex set $\bX$ in an ADMG $\cG$, let $\cG_\obX$ be the graph that is obtained from $\cG$ by deleting all incoming directed edges to $\bX$. We say that a vertex $V$ is an ancestor of vertex $W$ if there is a directed path from $V$ to $W$ (using only directed edges). The set of ancestors of $W$ with respect to the graph $\cG$ is denoted $\An(W, \cG)$.

\begin{definition}[hedge] \label{def:hedge}
    Let $\cG$ be an ADMG on a set of variables $\bV(\cG)$. For two disjoint vertex sets $\bX, \bY \subseteq \bV(\cG)$, a hedge for $\bX, \bY$ is given by two subgraphs $\cF, \cF' \subseteq \cG$ such that 
    \begin{enumerate}[label={(\arabic*)}]
        \item $\cF'$ is a subgraph of $\cF$;
        \item $\bX$ only occurs in $\cF$, i.e. $\bX \subseteq \bV(\cF \setminus \cF')$;
        \item $\cF$ and $\cF'$ are C-forests;
        \item $\cF$ and $\cF'$ have the same set $\bR$ of roots, and $\bR \subseteq \An(\bY, \cG_\obX)$. 
    \end{enumerate}
\end{definition}

\begin{lemma}[Causal ID algorithm \citep{shpitser08}] \label{lem:ID}
    Given a causal graph $\cG$ and disjoint vertex sets $\bX, \bY \subseteq \bV(\cG)$, then $P(\bY | \Do \bX)$ is unidentifiable if and only if there exists a hedge for $\bX, \bY$ in $\cG$. Moreover, there is a polynomial-time algorithm, the Causal ID algorithm~\ref{alg:ID}, that can decide whether or not a hedge exists for a given pair of sets $\bX, \bY$.
\end{lemma}

Algorithm~\ref{alg:ID} is a simplified version of the Causal ID algorithm, only giving a Boolean answer in case that the causal effect is identifiable. Here, $\cG[\bA]$ denotes the induced graph on a vertex set $\bA \subseteq \bV(\cG)$, and $C(\cG)$ denotes the set of C-components of $\cG$.

\begin{algorithm} 
    \caption{Causal ID algorithm (adapted from \citet{shpitser08})} \label{alg:ID}
     \begin{lstlisting}[style=pseudocode,mathescape,numbers=left,escapechar=|,columns=fullflexible,breaklines=true]
     Input: ADMG $\cG$ and disjoint sets $\bX, \bY \subseteq \bV(\cG)$.
     Output: True if $P(\bY | \Do \bX)$ is identifiable, hedge $(\cF, \cF')$ otherwise.
     if $\bX = \emptyset$, return True.
     if $\bV(\cG) \setminus \An(\bY) \neq \emptyset$, return $\ID(\cG[\An(\bY)], \bX \cap \An(\bY), \bY)$.
     let $\bW = \bV(\cG) \setminus (\bX \cup \An(\bY, \cG_{\obX}))$; if $\bW \neq \emptyset$, return $\ID(\cG, \bX \cup \bW, \bY)$.
     if $C(\cG \setminus \bX) = \{\cS_1, \dots, \cS_k\}$, return $\bigwedge_i \ID(\bV(\cG) \setminus \bV(\cS_i), \bV(\cS_i), \cG)$.
     if $C(\cG \setminus \bX) = \{\cS\}$:
     |\hspace{0.8cm}| if $C(\cG) = \{\cG)\}$, return the hedge $(\cG, \cG \cap \cS)$.
     |\hspace{0.8cm}| if $\exists \cS \subseteq \cS'$, such that $\cS' \in C(\cG)$, return $\ID(\cG[\bV(\cS')], \bX \cap \cS', \bY)$.
     |\hspace{0.8cm}| if $\cS \in C(\cG)$, return True.
    
     \end{lstlisting}  
 \end{algorithm}

\section{Results} \label{sec:results}

A priori, the Causal ID algorithm only works for causal graphs of finite size. However, note that in a periodic causal graph $\cG$ corresponding to a time series, directed edges are only going forward in time. First of all, this implies that the causal effect from $X_{i,t}$ on $X_{i', t'}$ is trivially identifiable for $t > t'$ (as we have $P(X_{i', t'} | \Do X_{i,t}) = P(X_{i', t'})$). Moreover, it is possible to compute identifiability of $P(X_{i', t'} | \Do X_{i,t})$ for $t \leq t'$, by running the Causal ID algorithm on the segment of the graph $\cG$ that includes all vertices up to time $t'$. To see this, note that any hedge for $X_{i,t}, X_{i', t'}$ can only consist of ancestors of $X_{i', t'}$. However, it is possible that there exists a hedge for $X_{i,t}, X_{i', t'}$ intersecting time layers strictly smaller than $t$, see Figure~\ref{fig:2}. This implies that even just identifying the causal effect between two single variables in the same layer $t$ a priori requires running the Causal ID algorithm on $\cG[0,t]$, which takes computing time that grows polynomially in $wt$. Moreover, using the Causal ID algorithm, one cannot answer the following question: Is the causal effect $P(X_{i', t'} | \Do X_{i,t})$ identifiable for all $t' > t$? 

Our results resolve both these problems. The following two statements are refinements of Theorem~\ref{thm:main} that we stated in the introduction. 

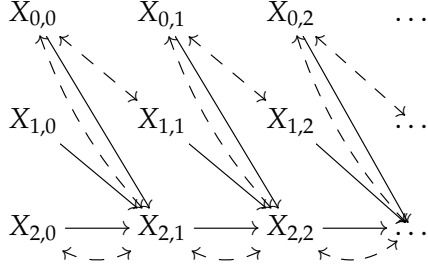
\begin{figure*}
 \[ \xymatrix{
X_{0,0} \ar@{<-->}[dr] \ar@{->}[ddr] \ar@/^-0.5pc/@{<-->}[ddr] &X_{0,1} \ar@{<-->}[dr] \ar@{->}[ddr] \ar@/^-0.5pc/@{<-->}[ddr] &X_{0,2} \ar@{<-->}[dr] \ar@{->}[ddr] \ar@/^-0.5pc/@{<-->}[ddr] & \ldots \\
X_{1,0} \ar@{->}[dr] &  X_{1,1} \ar@{->}[dr] &  X_{1,2} \ar@{->}[dr] & \ldots \\
X_{2,0} \ar@/^-1.0pc/@{<-->}[r] \ar@{->}[r]& X_{2,1} \ar@/^-1.0pc/@{<-->}[r] \ar@{->}[r] & X_{2,2} \ar@/^-1.0pc/@{<-->}[r] \ar@{->}[r]&  \ldots } \] 
\caption{The effect $P(X_{2,2} | \Do X_{1,1})$ is unidentifiable because of the unique hedge $(\cF, \cF')$ with $\bV(\cF) = \{X_{0,0}, X_{1,1}, X_{2,1}, X_{2,2}\}$ and $\bV(\cF') = \{X_{2,2}\}.$ Note that this unidentifiability cannot be detected when only looking at the layers from $X_{1,1}$ onward.}  \label{fig:2}
\end{figure*} 

\begin{proposition} \label{prop:past}
    Let $\cG$ be a periodic causal graph with width $w$ and latency $L$ and let $C = L \cdot 2^{Lw} \cdot (Lw + 1)^{2Lw+2}$. For subsets $\bX, \bY \subseteq \bV(\cG)$, we have that $P(\bY | \Do \bX)$ is unidentifiable if and only if there exists a hedge for $\bX, \bY$ in $\cG[\max(\tmin(\bX) - C, 0), \tmax(\bY)]$.
\end{proposition}

\begin{proposition} \label{prop:future}
    Let $\cG$ be a periodic causal graph with width $w$ and latency $L$ and let $C = L \cdot 2^{Lw} \cdot (Lw + 1)^{2Lw+2}$. Consider subsets $\bX, \bY \subseteq \bV(\cG)$ such that $\tmax(\bX) \leq \tmin(\bY)$. If $P(\bY | \Do \bX)$ is unidentifiable, then there exists a time shift $\Delta > \dist(\bX, \bY) - C$, such that $P(\bY_{-\Delta} | \Do \bX)$ is also unidentifiable.
\end{proposition}

Note that part~\ref{item:a} of Theorem~\ref{thm:main} follows from Proposition~\ref{prop:past} and part~\ref{item:b} of Theorem~\ref{thm:main} follows from repeatedly applying Proposition~\ref{prop:future}. In particular, these statements imply that whenever there are unidentifiable effects $P(\bY | \Do \bX)$ in the causal graph $\cG$, then they can be found by examining only a constant number of layers of $\cG$, independent of the time distance between $\bX$ and $\bY$ or the size of the past of the time series. Indeed, consider the following algorithm:

\begin{algorithm} 
    \caption{Decides identifiability of $P(\bY_{+\Delta} | \Do \bX)$ for all $\Delta \geq 0$.} \label{alg:main}
     \begin{lstlisting}[style=pseudocode,mathescape,numbers=left,escapechar=|,columns=fullflexible,breaklines=true]
     Input: Periodic causal graph $\cG$ of width $w$ and latency $L$, and disjoint sets $\bX, \bY \subseteq \bV(\cG)$ such that $\tmax(\bX) = \tmin(\bY)$.
     Output: True if $P(\bY_{+\Delta} | \Do \bX)$ is identifiable for all $\Delta \geq 0$, hedge $(\cF, \cF')$ for some $\bY_{+\Delta}, \bX$ otherwise.
     Let $C = L \cdot 2^{Lw} \cdot (Lw+1)^{2Lw+2}$.
     For $\Delta = 0, \dots, C-1$:
     |\hspace{0.8cm}| Let $B = \ID(\cG[\tmin(\bX) - C, \tmax(\bY_{+\Delta})], \bX, \bY_{+\Delta})$, if $B \neq $ True, return B.
     return True.
    
     \end{lstlisting}  
 \end{algorithm}

\begin{corollary}\label{cor:correctness}
Let $\cG$ be a periodic causal graph and $\bX, \bY \subseteq \bV(\cG)$ such that $\tmax(\bX) = \tmin(\bY)$. Then, Algorithm~\ref{alg:main} correctly computes whether $P(\bY_{+\Delta} | \Do \bX)$ is identifiable for all $\Delta \in \nset$.
\end{corollary}

Corollary~\ref{cor:correctness} follows directly from part~\ref{item:b} of Theorem~\ref{thm:main}. The value we are using for $C$ in Algorithm~\ref{alg:main} grows exponentially in the width $w$ and the latency $L$ of the causal graph. It would be of great interest to improve this dependence. The following statement shows that for some graphs, $C$ has to grow at least linearly in $w$. 

\begin{theorem} \label{thm:lowerbound}
    For infinitely many $w$, there exist periodic causal graphs $\cG_w$ with width $w$ and latency $1$, such that $P(Y | \Do X)$ is unidentifiable for some singletons $X, Y \in \bV(\cG)$ but the causal effects $P(Z | \Do X)$ are identifiable for all $Z \in \bV(\cG)$ satisfying $\dist(X, Z) \leq \frac{w}{3} - 1$.
\end{theorem}

\section{Proofs}

\subsection{Cutting hedges to show Propositions~\ref{prop:past}, \ref{prop:future}}

The proof of Proposition~\ref{prop:past} relies on the following idea: By Lemma~\ref{lem:ID} and time-directionality, we know that $P(\bY | \Do \bX)$ is unidentifiable if and only if there is a hedge for $\bX, \bY$ in $\cG[0, \tmax(\bY)]$. Now if this hedge exits the section $\cG[\tmin(\bX) - C, \tmax(\bY)]$ for a very large $C$, then it must be extremely stretched out. By periodicity of $\cG$, there must then be two columns of $\cG$ where the structure of the hedge is basically the same. We will show that cutting out the section of the graph between these two columns will result in a new hedge of smaller size. After doing this repeatedly, one eventually gets a hedge that is contained in $\cG[\tmin(\bX) - C, \tmax(\bY)]$. Proposition~\ref{prop:future} follows from the same idea, except that we are cutting out sections of the graph between $\bX$ and $\bY$ as long as $\bX$ and $\bY$ are too far away from each other. To formalize what we mean by the structure of the hedge in a given column, consider the following two definitions: 

\begin{definition}
Let $\cF \subseteq \cG$ be a subgraph of a periodic causal graph $\cG$. We say that vertices $X_{i,t}$ and $X_{i',t}$ are left-connected in $\cF$ if there exists a path of bidirected edges between them that goes only through vertices of $\cF$ that are of the form $X_{j,t'}, \; t' \leq t$. \end{definition}

\begin{definition} \label{def:functions}
For two sets of vertices $\bX, \bY \subseteq \cG$, let $(\cF, \cF')$ be a hedge for $\bX, \bY$. Let $\alpha_{\cF}(t)$ be the ordered partition of $\{0, \dots, w-1\}$ into $w+1$ blocks (some possibly empty), defined as follows: the first contains the row-indices of all vertices in $\bX_{*,t} \setminus \bV(\cF)$, while the subsequent blocks are the sets of row indices of the left-connected components of $\bX_{*,t} \cap \bV(\cF)$. If there are fewer than $w$ such blocks, then the final blocks of $\alpha_{\cF}(t)$ are empty. The blocks which correspond to left-connected components are sorted by the smallest index $i$ they contain. Likewise define $\alpha_{\cF'}(t)$. Finally, define $\beta(t) = \bX_{*, t} \cap \An(\bY, \cG_{\obX})$. 
\end{definition}

See Figure~\ref{left-conn} for an illustration of Definition~\ref{def:functions}. Next, we formalize what we mean by cutting out a segment of the graph. 

\begin{definition}
    We define the map $\Phi_{b,\Delta}: \cG \setminus \cG[b+1, b+\Delta] \to \cG$ to be the map that acts as the identity on $\cG[0, b]$ and translates $\cG[b+\Delta +1, \infty)$ by $-\Delta$ layers. That is, 
    \begin{align*}
    \Phi_{b, \Delta}(X_{i,t}) = \begin{cases}X_{i,t} \text{ if } t \leq b\\
    X_{i,t-\Delta} \text{ if } t > b+\Delta.\end{cases}
\end{align*}
\end{definition}

\begin{figure*}
\centering
\includegraphics[scale=1]{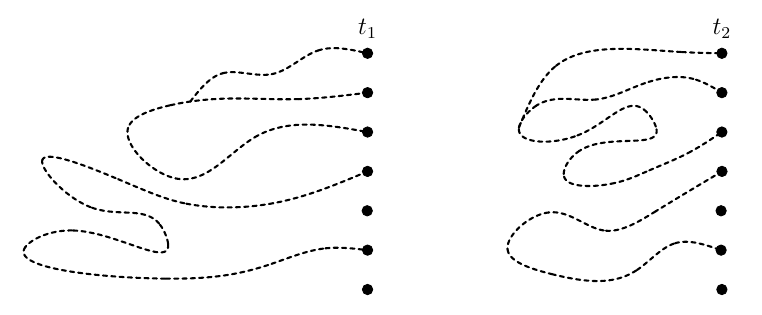}
\caption{Two levels $t_1$ and $t_2$ with the same left-connected structure. Assuming that indices range from $0$ to $6$ from top to bottom, and all vertices and paths are in $\cF$, we get $\alpha_\cF(t_1) = \alpha_\cF(t_2) = (\emptyset, \{1,2,3\}, \{4,6\}, \{5\}, \{7\})$.} \label{left-conn} 
\end{figure*}

The map $\Phi$ can be thought of as an operation that cuts the part $\cG[b+1, b+\Delta]$ out of the graph $\cG$ and then rewires the loose ends at layers $b$ and $b+\Delta+1$ together according to the periodic graph structure of $\cG$. The key property of $\Phi$ is the following: due to the periodic structure of $\cG$, whenever $(V, W)$ is a directed (resp. bidirected) edge for $V, W \in \bV(\cG \setminus \cG[b+1, b+\Delta])$, then $(\Phi(V), \Phi(W))$ is also a directed (resp. bidirected) edge in $\cG$ (and the converse almost holds, except if $V \in \bX_{*, b}$ and $W \in \bX_{*, b+\Delta+1})$. Here, we are assuming that $\cG$ is of latency $1$, and we will cover graphs with higher latency later. Our main technical result states that cutting at two layers on which the values of $\alpha_\cF, \alpha_{\cF'},$ and $\beta$ align for a hedge $(\cF, \cF')$ creates a new (and smaller) hedge: 

\begin{lemma}\label{lem:technical}
    Let $\cG$ be a periodic causal graph of latency $L=1$, and let $(\cF, \cF')$ be a hedge for the vertex sets $\bX, \bY \subseteq \cG$. Assume $\cG[b+1, b+\Delta] \cap (\bX \cup \bY) = \emptyset$, and we have $\alpha_\cF(b) = \alpha_\cF(b+\Delta), \alpha_{\cF'}(b) = \alpha_{\cF'}(b+\Delta),$ and $\beta(b) = \beta(b + \Delta)$. Then, $(\Phi_{b, \Delta}(\cF \setminus \cG[b+1, b+\Delta]), \Phi_{b, \Delta}(\cF' \setminus \cG[b+1, b+\Delta]))$ is a hedge for $\Phi_{b, \Delta}(\bX), \Phi_{b, \Delta}(\bY)$.
\end{lemma}

\begin{proof}
Fix $b$ and $\Delta$ satisfying the hypotheses of Lemma~\ref{lem:technical}. To simplify notation, let 
\begin{align*}
    \Phi &= \Phi_{b, \Delta},\\
    \tcF &= \Phi(\cF \setminus \cG[b+1, b+\Delta]),\\
    \tcF' &= \Phi(\cF' \setminus \cG[b+1, b+\Delta]).
\end{align*}

Following Definition~\ref{def:hedge}, it suffices to prove that the following statements hold:

\begin{enumerate}[label=(\alph*)]
    \item $\tcF' \subseteq \tcF$; \label{prop:0}
    \item $\Phi(\bX) \subseteq \bV(\tcF \setminus \tcF')$; \label{prop:a}
    \item $\tcF$ and $\tcF'$ are C-connected.\label{prop:b}
    \item Every vertex in $\tcF$ has at most one child. \label{prop:c}
    \item $\tcF$ and $\tcF'$ have the same set $\tbR$ of roots, and $\tbR \subseteq \An(\Phi(\bY), \cG_{\overline{\Phi(\cX)}})$. \label{prop:d}
\end{enumerate}

Statements \ref{prop:0} and \ref{prop:a} follow straight from the definitions and the fact that $\Phi$ is injective. To show Statement \ref{prop:b}, we first prove that any two vertices $X_{i,b}, X_{j,b} \in \bV(\tcF)$ on layer $b$ are C-connected in $\tcF$. Since $\alpha_\cF(b) = \alpha_\cF(b+\Delta)$, we know that $X_{i,b+\Delta}, X_{j, b+\Delta} \in \bV(\cF)$, so there must be a path of bidirected edges in $\cG$ that connects these two vertices and only goes through vertices of $\cF$. Suppose this path hits layer $b+\Delta$ exactly $m$ times at the vertices $X_{k_1, b+\Delta}, \dots, X_{k_m, b+\Delta}$ in this order, where $X_{k_1, b+\Delta} = X_{i, b+\Delta}$ and $X_{k_m, b+\Delta} = X_{j, b+\Delta}$. We claim that then there is a path of bidirectional edges connecting $X_{k_1, b}, \dots, X_{k_m, b}$ in $\tcF$. Observe that, since $\cG$ has latency $1$, for any $\ell \in \{1, \dots, m-1\}$ the path from $X_{k_\ell, b+\Delta}$ to $X_{k_{\ell+1}, b+\Delta}$ in $\cF$ is either completely contained in $\cG[0, b+\Delta]$ or in $\cG[b+\Delta, \infty)$. In the first case, $X_{k_\ell, b+\Delta}$ and $X_{k_{\ell+1}, b+\Delta}$ are left-connected in $\cF$. Hence, $X_{k_\ell, b}$ and $X_{k_{\ell+1}, b}$ are also left-connected in $\cF$, but to the left of layer $b$, we have $\cF \cap \cG[0,b] = \tcF \cap \cG[0,b]$, so both vertices are also left-connected in $\tcF$. In the second case, if the path from $X_{k_\ell, b+\Delta}$ to $X_{k_{\ell+1}, b+\Delta}$ in $\cF$ stays to the right of layer $b+\Delta$, then we can just translate it by $-\Delta$ layers to get a path from $X_{k_\ell, b}$ to $X_{k_{\ell+1}, b}$ in $\tcF$. Hence, $X_{i, b}$ and $X_{j,b}$ are C-connected in $\tcF$. Now we show that any two vertices $\widetilde{A}, \widetilde{B} \in \bV(\tcF)$ are C-connected in $\tcF$. Consider the path $\pi$ of bidirectional edges between $A = \Phi^{-1}(\widetilde{A})$ and $B = \Phi^{-1}(\widetilde{B})$ that goes through $\cF$. If the path $\pi$ never hits layer $b$ or $b+\Delta$, then $\Phi(\pi)$ is a path between $\widetilde{A}$ and $\widetilde{B}$ (it either stays the same or it gets translated as a whole). Otherwise, let $A'$ and $B'$ be the first vertices that are on layer $b$ or $b+\Delta$ and get hit by the path $\pi$ when starting from $A$ or $B$ respectively. Let $\pi_1$ denote the path from $A$ to $A'$ and $\pi_2$ denote the path from $B$ to $B'$. Then $\Phi(\pi_1)$ is a path from $\widetilde{A}$ to $\Phi(A')$ in $\tcF$ and $\Phi(\pi_2)$ is a path from $\widetilde{B}$ to $\Phi(B')$ in $\tcF$. Both, $\Phi(A'), \Phi(B') \in \bV(\tcF)$ are in layer $b$, so they must also be C-connected in $\tcF$. Hence, $\widetilde{A}$ and $\widetilde{B}$ are C-connected in $\tcF$. This shows $\tcF$ is C-connected and the same statement follows for $\tcF'$ analogously.

To show Statement \ref{prop:c}, assume that $\widetilde{A} \in \bV(\tcF)$ has at least two children $\widetilde{B}, \widetilde{C}$ in $\tcF$. Let $A = \Phi^{-1}(\widetilde{A})$, $B = \Phi^{-1}(\widetilde{B})$, and $C = \Phi^{-1}(\widetilde{C})$. We know that $A,B,C \in \bV(\cF)$. Moreover, if $A, B, C$ are fully contained in $\cG[0,b]$ or $A, B, C$ are fully contained in $\cG[b+\Delta+1, \infty)$, then $B,C$ are children of $A$, which is a contradiction to $\cF$ being part of a hedge. Hence, we must have $A \in \bX_{*, b}$ and $B, C \in \bX_{*, b+\Delta+1}$. Let $A_{+\Delta} \in \bX_{b+\Delta}$ be the vertex $A$ shifted $\Delta$ layers to the right. Then, $B$ and $C$ are children of $A_{+\Delta}$ but since $\alpha_\cF(b) = \alpha_cF(b+\Delta)$, we have $A_{+\Delta} \in \bV(\cF)$. This is again a contradiction to $\cF$ being part of a valid hedge and shows the statement. 

To prove Statement \ref{prop:d}, observe first that one can show $\tcF$ and $\tcF'$ have the same set $\tbR$ of roots using the same logic as the proof of Statement~\ref{prop:c}: Whenever there is a vertex $\widetilde{R}$, which is a root of $\tcF$ but not of $\tcF'$ or vice-versa, one can map it back (and potentially switch from layer $b$ to $b+\Delta$) to find a vertex $R$, which is a root of $\cF$ but not of $\cF'$ or vice-versa, leading to a contradiction. Hence, it only remains to show that $\tbR \subseteq \An(\Phi(\bY), \cG_{\overline{\Phi(\bX)}})$. Consider a vertex $\widetilde{R} \in \tbR$ and its preimage $R = \Phi^{-1}(\widetilde{R}) \in \bR$, where $\bR$ is the set of roots of $\cF, \cF'$. We know that there is a directed path $\pi$ in $\cG_{\overline{X}}$ from $R$ to some vertex $Y \in \bY$. If $\pi$ does not hit layer $b$, then it cannot hit the region $\cG[b,b+\Delta]$ either, so $\Phi(\pi)$ is a directed path in $\cG_{\overline{\Phi(\bX)}}$ from $\widetilde{R}$ to $\Phi(Y) \in \Phi(\bY)$. Otherwise, let $X_{i,b}$ be the first vertex that $\pi$ hits in layer $b$. Since $X_{i,b}$ is an ancestor of $\bY$ in $\cG_{\overline{X}}$ and $\beta(b) = \beta(b+\Delta)$, we know that there is a directed path $\tau$ in $\cG_{\overline{X}}$ from the vertex $X_{i,b+\Delta}$ to some vertex $Y'$ in $\bY$. Hence, following the path $\Phi(\pi)$ from $\widetilde{R}$ to $X_{i,b}$ and the path $\Phi(\tau)$ from $X_{i,b}$ to $\Phi(Y')$ results in a path from $\widetilde{R}$ to $\Phi(\bY)$ in $\cG_{\overline{\Phi(\bX)}}$. This completes the proof of Lemma~\ref{lem:technical}.
\end{proof}

Now we show how Lemma~\ref{lem:technical} implies Proposition~\ref{prop:past}.

\begin{proof}(of Proposition~\ref{prop:past}).
    Let $\cG$ be a periodic causal graph of width $w$ and latency $L$, and $\bX, \bY \subseteq \bV(\cG)$. Wlog we may assume that $\tmin(\bX) \leq \tmin(\bY)$. By Lemma~\ref{lem:ID}, it suffices to show that if there exists a hedge for $\bX, \bY$, then there is also a hedge for $\bX, \bY$ in $\cG[\tmin(\bX) - C, \tmax(\bY)]$. Note that, by time-directionality, any hedge for $\bX, \bY$ is contained in $\cG[0, \tmax(\bY)]$. Now, to be able to apply Lemma~\ref{lem:technical}, we transform $\cG$ into a periodic causal graph $\cH$ of width $Lw$ and latency $1$, by simply relabeling the vertices of $\cG$ such that layer $t$ of $\cH$ contains the vertices of the layers $Lt, Lt+1, \dots, Lt +t-1$ of $\cG$ (that is, we aggregate each $L$ consecutive time steps of $\cG$). Let $\bX', \bY'$ be the sets in $\bV(\cH)$ corresponding to $\bX, \bY$. Since the graphs $\cG$ and $\cH$ are isomorphic, if there exists a hedge for $\bX, \bY$ in $\cG$, then there exists a hedge for $\bX', \bY'$ in $\cH$. Let $(\cF, \cF')$ such a hedge with maximal $\tmin(\cF)$. We claim that $\tmin(\bX') - \tmin(\bV(\cF)) \leq C/L = 2^{Lw}(Lw+1)^{2Lw+2}$. 

    Suppose not, and consider the functions $\alpha_\cF(t)$, $\alpha_{\cF'}(t)$ and $\beta(t)$. Since $\cH$ has width $Lw$, both the partitions $\alpha_\cF(t)$ and $\alpha_{\cF'}(t)$ can attain at most $(Lw+1)^{Lw+1}$ values, while the function $\beta(t)$ can attain at most $2^{Lw}$ values. Hence, by the pigeonhole principle, there must exist some layers $b$ and $b + \Delta$, such that $\tmin(\cF) \leq b < b+\Delta < \tmin(\bX')$ and $\alpha_\cF(b) = \alpha_\cF(b+\Delta)$, $\alpha_{\cF'}(b) = \alpha_{\cF'}(b+\Delta)$, and $\beta(b) = \beta(b+\Delta)$. Now, by Lemma~\ref{lem:technical}, there exists a hedge $(\tcF, \tcF')$ for $\bX'_{-\Delta}, \bY'_{-\Delta}$ that has $\tmin(\tcF) = \tmin(\cF)$. After shifting this hedge by $\Delta$ layers to the right, we get another hedge for $\bX', \bY'$ with strictly larger minimum layer than $\tmin(\cF)$. This is a contradiction to our minimality assumption, and hence $(\cF, \cF')$ must have been contained in $\cH[\tmin(\bX') - C/L, \tmax(\bY')]$. After reversing the aggregation of layers, we find a corresponding hedge for $\bX, \bY$ contained in $\cG[\tmin(\bX) - C, \tmax(\bY)]$, which completes the proof. 
    
\end{proof}

The proof of Proposition~\ref{prop:future} is analogous, except that we are using the pidgeonhole principle on the layers between $\tmax(\bX)$ and $\tmin(\bY)$ to apply Lemma~\ref{lem:technical}.

\subsection{Constructing examples to show Theorem \ref{thm:lowerbound}}

The goal of this section is to construct periodic causal models $\cG_w$ that satisfy the statement of Theorem~\ref{thm:lowerbound}. For this, it is enough to define an appropriate acyclic graph $\cA_w$ on $\bX_{*,0}\cup \bX_{*,1}$. We can extend $\cA_w$ to a periodic, latency $1$ graph $\cG_w$ by including a directed (resp.\ bidirected) edge $(X_{i,t},X_{i',t'})$ in $\cG_w$ if $(X_{i,0},X_{i',t'-t})$ is in $\cA_w$. 

Specifically, let $w = 3k+1$ and define $\cG_w$ to have (see Figure~\ref{Hw-fig}):
\begin{itemize}
    \item Directed edges from every $X_{i,t}$ to $X_{i, t+1}$ and to $X_{i+3, t+1}$.
    \item Bidirected edges from every $X_{i,t}$ to $X_{i+1, t+1}$ and to $X_{i+2, t+1}$.
    \item In both cases the index $i$ ranges from $0$ to $w-1$ and we use addition modulo $w$. 
\end{itemize}

usetikzlibrary{arrows.meta}

\begin{figure}[ht]
\centering
\begin{tikzpicture}[line cap=round,line join=round,>=Latex,x=1cm,y=1cm, scale=0.8]
\draw [->, line width=1pt,] (-4,2) -- (0,2);
\draw [->, line width=1pt] (-4,2) -- (0,-4);
\draw [line width=1pt, dash pattern=on 3pt off 3pt] (-4,2) -- (0,0);
\draw [line width=1pt, dash pattern=on 3pt off 3pt] (-4,2) -- (0,-2);
\draw [line width=1pt, dash pattern=on 3pt off 3pt] (-4,0) -- (0,-2);
\draw [line width=1pt, dash pattern=on 3pt off 3pt] (-4,0) -- (0,-4);
\draw [line width=1pt, dash pattern=on 3pt off 3pt] (-4,-2) -- (0,-4);
\draw [line width=1pt, dash pattern=on 3pt off 3pt] (-4,-2) -- (0,-6);
\draw [line width=1pt, dash pattern=on 3pt off 3pt] (-4,-4) -- (0,-6);
\draw [line width=1pt, dash pattern=on 3pt off 3pt] (-4,-4) -- (0,-8);
\draw [line width=1pt, dash pattern=on 3pt off 3pt] (-4,-6) -- (0,-8);
\draw [line width=1pt, dash pattern=on 3pt off 3pt] (-4,-6) -- (0,-10);
\draw [line width=1pt, dash pattern=on 3pt off 3pt] (-4,-8) -- (0,-10);
\draw [line width=1pt, dash pattern=on 3pt off 3pt] (-4,-8) -- (0,2);
\draw [line width=1pt, dash pattern=on 3pt off 3pt] (-4,-10) -- (0,2);
\draw [line width=1pt, dash pattern=on 3pt off 3pt] (-4,-10) -- (0,0);
\draw [->, line width=1pt] (-4,-10)-- (0,-2);
\draw [->, line width=1pt] (-4,-8)-- (0,0);
\draw [->, line width=1pt] (-4,-6)-- (0,2);
\draw [->, line width=1pt] (-4,-4)-- (0,-10);
\draw [->, line width=1pt] (-4,-2)-- (0,-8);
\draw [->, line width=1pt] (-4,0)-- (0,-6);
\draw [->, line width=1pt] (-4,-10)-- (0,-10);
\draw [->, line width=1pt] (-4,-8)-- (0,-8);
\draw [->, line width=1pt] (-4,-6)-- (0,-6);
\draw [->, line width=1pt] (-4,-4)-- (0,-4);
\draw [->, line width=1pt] (-4,-2)-- (0,-2);
\draw [->, line width=1pt] (-4,0)-- (0,0);
\begin{scriptsize}
\draw [fill=black] (-4,2) circle (3pt);
\draw[color=black] (-4.6,2) node {\normalsize $X_{0,0}$};
\draw [fill=black] (0,2) circle (3pt);
\draw[color=black] (0.6,2) node {\normalsize $X_{0,1}$};
\draw [fill=black] (-4,0) circle (3pt);
\draw[color=black] (-4.6,0) node {\normalsize $X_{1,0}$};
\draw [fill=black] (0,0) circle (3pt);
\draw[color=black] (0.6,0) node {\normalsize $X_{1,1}$};
\draw [fill=black] (-4,-2) circle (3pt);
\draw[color=black] (-4.6,-2) node {\normalsize $X_{2,0}$};
\draw [fill=black] (0,-2) circle (3pt);
\draw[color=black] (0.6,-2) node {\normalsize $X_{2,1}$};
\draw [fill=black] (-4,-4) circle (3pt);
\draw[color=black] (-4.6,-4) node {\normalsize $X_{3,0}$};
\draw [fill=black] (0,-4) circle (3pt);
\draw[color=black] (0.6,-4) node {\normalsize $X_{3,1}$};
\draw [fill=black] (-4,-6) circle (3pt);
\draw[color=black] (-4.6,-6) node {\normalsize $X_{4,0}$};
\draw [fill=black] (0,-6) circle (3pt);
\draw[color=black] (0.6,-6) node {\normalsize $X_{4,1}$};
\draw [fill=black] (-4,-8) circle (3pt);
\draw[color=black] (-4.6,-8) node {\normalsize $X_{5,0}$};
\draw [fill=black] (0,-8) circle (3pt);
\draw[color=black] (0.6,-8) node {\normalsize $X_{5,1}$};
\draw [fill=black] (-4,-10) circle (3pt);
\draw[color=black] (-4.6,-10) node {\normalsize $X_{6,0}$};
\draw [fill=black] (0,-10) circle (3pt);
\draw[color=black] (0.6,-10) node {\normalsize $X_{6,1}$};
\end{scriptsize}
\end{tikzpicture}

\caption{The graph $\cA_7$.} \label{Hw-fig} 
\end{figure}
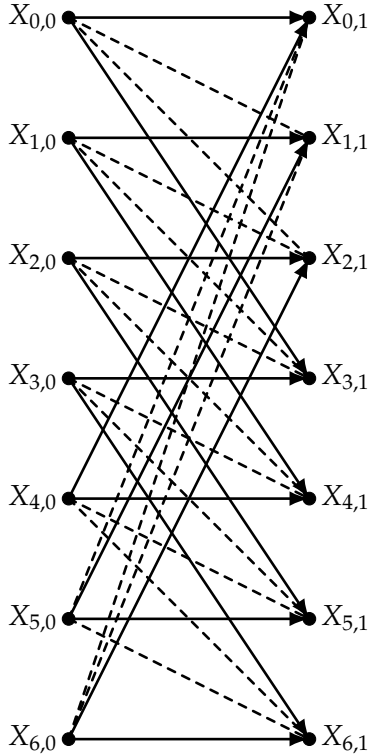
 
First, we claim that in $\cG_w$, the causal effect $P(X_{w-1,w-2}|\text{ do } X_{0,0})$ is unidentifiable. 
To show this, we construct a hedge for $X_{0,0}, X_{w-1,w-2}$. Our hedge consists of the following sets of vertices (see Figure~\ref{hedge-fig})
\begin{align*}
 \cF = \{&X_{0,0}, X_{0,1}, X_{1,1}, X_{w-3,0},X_{w-2,0}, X_{w-1,0} , X_{w-1,1} \}\\
 \cF' = \{&X_{0,1}, X_{1,1}, X_{w-3,0}, X_{w-2,0}, X_{w-1,0}, X_{w-1,1} \}    
\end{align*}

\begin{figure*}[ht]
\centering
\begin{tikzpicture}[line cap=round,line join=round,>=triangle 45, scale=2.7]
\draw [->, line width=1pt] (0,3) -- (1,3);
\draw [->, line width=1pt,color=aqaqaq] (1,3) -- (2,3);
\draw [->, line width=1pt] (0,0) -- (1,0);
\draw [->, line width=1pt,color=aqaqaq] (1,3) -- (2,1.5);
\draw [->, line width=1pt] (0,0.5) -- (1,2.5);
\draw [->, line width=1pt] (0,1) -- (1,3);
\draw [line width=1pt,dash pattern=on 3pt off 3pt] (0,3)-- (1,2.5);
\draw [line width=1pt,dash pattern=on 3pt off 3pt] (0,1)-- (1,0);
\draw [line width=1pt,dash pattern=on 3pt off 3pt] (0,0.5)-- (1,3);
\draw [line width=1pt,dash pattern=on 3pt off 3pt] (0,0)-- (1,3);
\draw [line width=1pt,dash pattern=on 3pt off 3pt] (0,0)-- (1,2.5);
\draw [->, line width=1pt,color=aqaqaq] (1,2.5) -- (2,1);
\draw [->, line width=1pt,color=aqaqaq] (3,3) -- (4,1.5);
\draw [->, line width=1pt,color=aqaqaq] (2,3) -- (3,1.5);
\draw [->, line width=1pt,color=aqaqaq] (4,1.5) -- (5,0);
\draw [->, line width=1pt,color=aqaqaq] (3,1.5) -- (4,0);
\draw [->, line width=1pt,color=aqaqaq] (2,1.5) -- (3,0);
\draw [->, line width=1pt,color=aqaqaq] (1,0) -- (2,0);
\draw [->, line width=1pt,color=aqaqaq] (2,0) -- (3,0);
\draw [->, line width=1pt,color=aqaqaq] (3,0) -- (4,0);
\draw [->, line width=1pt,color=aqaqaq] (4,0) -- (5,0);
\draw [->, line width=1pt,color=aqaqaq] (2,1) -- (3,3);
\draw [line width=1pt,dash pattern=on 3pt off 3pt] (0,0.5)-- (1,0);
\begin{scriptsize}
\draw [fill=black] (0,3) circle (1.5pt);
\draw[color=black] (0,3.2) node {\normalsize $X_{0,0}$};
\draw [fill=black] (1,3) circle (1.5pt);
\draw[color=black] (1,3.2) node {\normalsize $X_{0,1}$};
\draw [fill=black] (1,2.5) circle (1.5pt);
\draw[color=black] (1.0684804377189951,2.641852877593049) node {\normalsize $X_{1,1}$};
\draw [color=black] (2,3) circle (1.5pt);
\draw [fill=black] (0,1) circle (1.5pt);
\draw[color=black] (-0.35,1) node {\normalsize $X_{w-3,0}$};
\draw [fill=black] (0,0.5) circle (1.5pt);
\draw[color=black] (-0.35,0.5) node {\normalsize $X_{w-2,0}$};
\draw [fill=black] (0,0) circle (1.5pt);
\draw[color=black] (-0.35,0) node {\normalsize $X_{w-1,0}$};
\draw [fill=black] (1,0) circle (1.5pt);
\draw[color=black] (1.1025322189938356,-0.2) node {\normalsize $X_{w-1,1}$};
\draw [color=black] (2,1.5) circle (1.5pt);
\draw [color=black] (2,1) circle (1.5pt);
\draw [color=black] (3,1.5) circle (1.5pt);
\draw [color=black] (4,1.5) circle (1.5pt);
\draw [color=black] (3,3) circle (1.5pt);
\draw [color=black] (2,0) circle (1.5pt);
\draw [color=black] (3,0) circle (1.5pt);
\draw [color=black] (4,0) circle (1.5pt);
\draw [color=black] (5,0) circle (1.5pt);
\draw[color=black] (5.1,-0.2) node {\normalsize $X_{w-1, 1+2(w-1)/3} $};
\end{scriptsize}
\end{tikzpicture}

\caption{Hedge for $w=7$} \label{hedge-fig} \end{figure*}
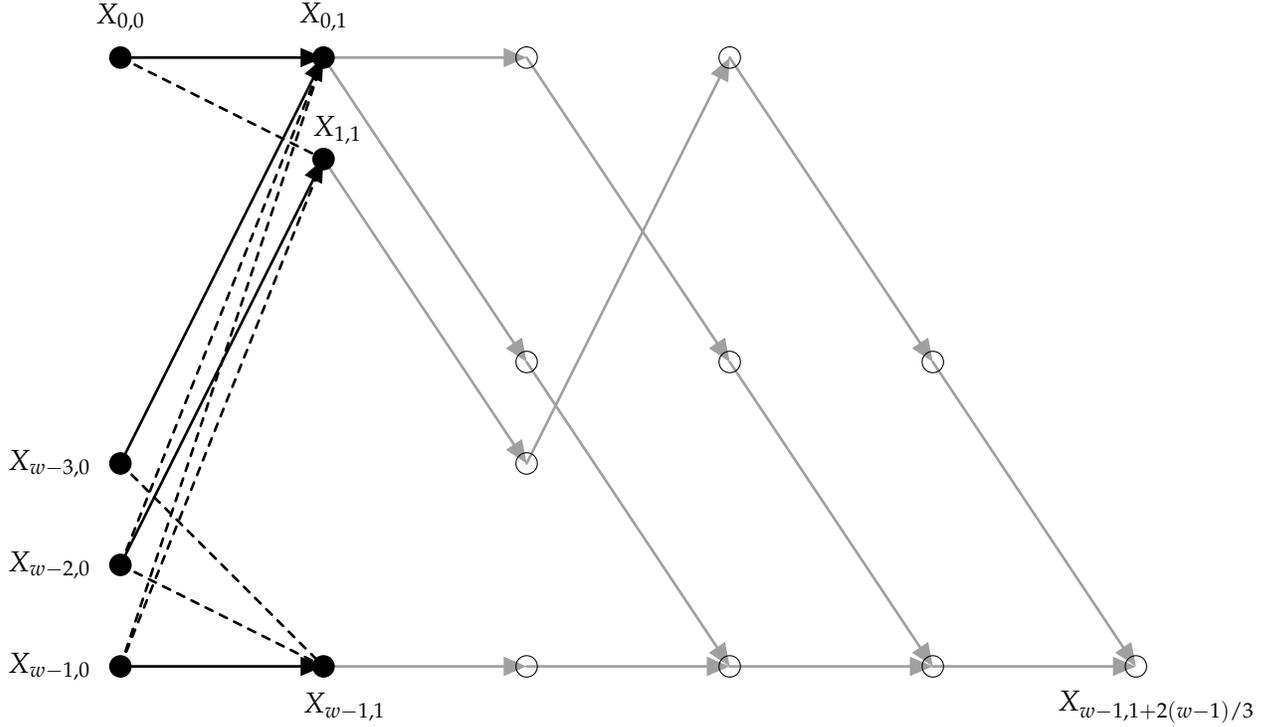

Both $\cF$ and $\cF'$ have the same set of roots $\bR = \{X_{0,1}, X_{1,1}, X_{w-1,1} \}.$ We now claim that there are directed paths from each of these roots to our effect vertex $X_{w-1,w-2}.$ Clearly, since we have directed edges from $X_{i,j}$ to $X_{i,j+1},$ there is a (``horizontal'') directed path from $X_{w-1,1}$ to $X_{w-1,w-2}$. Additionally, since $w$ is of the form $3k+1,$ $X_{0,1}$ has a directed (``diagonal'') path to the directed chain of vertices in row $w-1$ (meeting it at the vertex $X_{w-1,(w-1)/3+1}$), and therefore, it has a directed path to $X_{w-1,w-2}.$ Finally, from $X_{1,1}$ there is a diagonal path of length $2(w-1)/3$ which wraps around (vertically) until it reaches the vertex $X_{w-1,2(w-1)/3+1}$, from which a path continues horizontally to $X_{w-1,w-2}$. (This last wrap-around path is key to the construction of the hedge. Indeed, if the graph were constructed with the same edge types but without vertical coordinates being taken modulo $w$,  $\Pr(X_{j,t} \vert  \text{ do } X_{0,0})$ would be identifiable for all $X_{j,t}$.) Finally, it is easy to verify that the bidirected edges make both $\cF$ and $\cF'$ $C$-connected; this is illustrated for $w=7$ in Figure~\ref{hedge-fig}.

Due to the existence of this hedge, $\Pr(X_{w-1,w-2} \vert \text{ do } X_{0,0})$ is not identifiable. We now prove that for every $t \leq \frac{w}{3}-1$ and every $j$, $P(X_{j,t}|\text{ do } X_{0,0})$ is identifiable.

First, assume $j$ is not divisible by $3$. In this case, there is no directed path from $X_{0,0}$ to $X_{j,t}$. This can be seen as all directed edges vertically descend by either $0$ rows or $3$ rows, so all descendants of $X_{0,0}$ have a first index that is divisible by $3$ up to time $\frac{w-1}{3}$. We deduce that the causal effect $P(X_{j,t} | \Do X_{0,0})$ is trivial, and hence identifiable. 

It remains to consider the vertices $X_{j,t}$ with $3\vert j$. The hedge for such a vertex must consist solely of ancestors of $X_{j,t}$, which are vertices of the form $X_{j',t'}$ with $t'\leq t$, and $j'$ satisfying either $0\leq j' \leq j$, or else $w+j-3(t-t')\leq j' \leq w-1$.

As a consequence, there are two rows, namely the rows indexed $j+1$ and $j+2$, which cannot contain the vertices of a hedge for $X_{j,t}$. Therefore (after re-indexing the rows to start from $j+1$ and end at $j$), we can consider all directed and bidirected edges among hedge vertices to occur within a vertical segment, all directed edges skipping $0$ or $3$ rows, and all bidirected edges skipping $1$ or $2$ rows. The bidirected edges therefore can never connect two hedge vertices. The larger part of the hedge, $\cF$, must contain more than one vertex. It therefore cannot be $C$-connected. Consequently, there is no hedge, so $P(X_{j,t} | \Do X_{0,0})$ is identifiable. This proves Theorem~\ref{thm:lowerbound}.

\section{Discussion} \label{sec:discussion}
There remains an exponential gap between our upper and lower bounds on the number $C$ of layers that are necessary to consider when deciding identifiability of a causal effect in a time series graph. We suspect that the lower bound is closer to the truth, but further work is needed to confirm this suspicion. If this can be done, it will make the result much more attractive computationally. Note that our upper bound on $C$ is for graphs of generic periodic structure. It would be an interesting question of whether the bound can be improved given certain structural conditions on the periodic causal graph.

Our setting is also generic in that we do not impose time-invariance on the functional dependencies of the underlying causal system. This is why our results focus on deciding whether causal effects are identifiable rather than identifying them, as we do not expect a non-trivial speed-up for the latter (at least not in the sense of part~\ref{item:b} of Theorem~\ref{thm:main}). Efficient identification in time series graphs, potentially under the additional assumption of time-invariant dynamics, is therefore still a subject for further research.

A separate topic of interest is the condition number of the Causal ID mapping in periodic graphs (see \citet{SchSriv16} for a definition). When one tries to translate an in-principle ``perfect statistics'' result into an actual application with empirical data, the sample size scales roughly proportionally to this condition number. It is known that, in general, the condition number can vary widely between different graphs: it can be exponential in the size of the graph~\citep{SchSriv16}, but on the other hand is also known to scale only as $\exp(c \log c)$ if the $C$-components of the graph are of size bounded by $c$~\citep{gordonKSS21}. It seems quite likely that the periodic structure should impose further constraints that will help to upper bound the condition number, but this is at present an entirely open question.

\acks{This work was supported by NSF grant CCF-2321079.}

\bibliography{refs}

\end{document}